\documentclass[letterpaper]{article} 
\usepackage{aaai18}  
\usepackage{times}  
\usepackage{helvet}  
\usepackage{courier}  
\usepackage{url}  
\usepackage{graphicx}  

\usepackage{times}
\usepackage{xcolor}
\usepackage{soul}
\usepackage[utf8]{inputenc}
\usepackage[small]{caption}
\usepackage{enumitem}
\usepackage{latexsym} 
\usepackage{changebar}

\usepackage{amsmath,amssymb,amsthm} 
\usepackage{url}
\usepackage{tabularx, booktabs, caption, makecell, ragged2e}

\usepackage {tikz}
\usetikzlibrary {positioning}
\usepackage{subfig}
\newcolumntype{Y}{>{\RaggedRight\arraybackslash}X}
\usepackage{listings}
\title{
Hybrid Temporal Situation Calculus}
\iftrue
\author{
Vitaliy Batusov$^1$, 
Giuseppe De Giacomo$^2$, 
Mikhail Soutchanski$^3$, 
\\ 
$^1$ York University, Toronto, ON, Canada  \\
$^2$ Sapienza Universit\`a di Roma, Italy \\
$^3$ Ryerson University, Toronto, ON, Canada  \\
vbatusov@cse.yorku.ca,
degiacomo@dis.uniroma1.it,
mes@scs.ryerson.ca
}
\fi

\newcommand{\commentout}[1]{}

\newcommand{\liff}{\leftrightarrow}

\newcommand{\eq}{\!=\!}
\newcommand{\nneq}{\!\neq\!}
\newcommand{\ggeq}{\!\geq\!}
\newcommand{\lleq}{\!\leq\!}

\newcommand{\init}{_\text{\emph{init}}}

\newtheorem{theorem}{Theorem}
\newtheorem{lemma}{Lemma}
\theoremstyle{definition}
\newtheorem*{example}{Example}

\nochangebars
\begin{document}
\nocopyright 
\maketitle


\begin{abstract}
The ability to model continuous change in Reiter's temporal situation calculus action theories has attracted a lot of interest. 
In this paper, we propose a new development of his approach, which is directly inspired by hybrid systems in control theory. Specifically, while keeping the foundations of Reiter's axiomatization, we propose an elegant extension of his approach by adding a time argument to all fluents that represent continuous change. Thereby, we insure that change can happen not only because of actions, but also due to the passage of time. We present a systematic methodology to derive, from simple premises, a new group of axioms which specify how continuous fluents change over time within a situation. We study regression for our new temporal basic action theories and demonstrate what reasoning problems can be solved.
Finally, we formally show \cbstart that \cbend our temporal basic action theories indeed capture hybrid automata.
\end{abstract}

\section{Introduction}

Adding time and continuous change to Situation Calculus action theories
has attracted a lot of interest over the years.  A seminal book
\cite{reiter}, refining the ideas of \cite{PintoPhD,pinto1995},
extends situation calculus (SC) with continuous time.
For each continuous process, there is an action that initiates the process
at a moment of time, and there is an instantaneous action that terminates it.
A basic tenet of Reiter's temporal SC is that all changes in the world, 
including continuous processes such as a vehicle driving in a city or water flowing down a pipe, are 
the result of named \cbstart discrete \cbend actions. Consequently, in his temporal extension of SC, 
fluents remain atemporal, while each instantaneous action acquires a time argument.
As a side effect of this design choice, continuously varying quantities 
do not attain values until the occurrence of a time-stamped action. 
\commentout{    
For example, in Newtonian physics, suppose a player kicks a football,
sending it on a ballistic trajectory. The question might be: given the
vector of initial velocity, when will the ball reach the peak of its
trajectory? or: when will the ball be within 10\% of said peak?  In order to answer these questions, natural (exogenous) actions are introduced to be able to deem the moment of interest for the query. \cbstart In order to \cbend  directly answer such questions, one needs the ability to formulate
queries about the height of the ball at arbitrary time-points, which
is impossible without an explicit time variable in the query. 
}       
\cbstart
For example, in Newtonian physics, suppose a player kicks a football,
sending it on a ballistic trajectory. The questions might be, e.g., given
the vector of initial velocity, when will the ball reach the peak of its
trajectory? Or, when will the ball be within 10\% of said peak?
In order to answer these questions either a natural, or an exogenous
(respectively) action, depending on a query, has to be executed to deem 
the moment of interest for the query. In other words,
before one can answer such questions, one needs the ability
to formulate queries about the height of the ball at arbitrary time-points,
which is not directly possible without an explicit action with a time argument,
if a query is formed from atemporal fluents.  
\cbend
In Reiter's temporal SC, to query about the values of physical quantities
in between the actions (agent's or natural), one could opt for an
auxiliary exogenous action $watch(t)$ \cite{mes99}, whose purpose is
to fix a time-point $t$ to a situation when it occurs, and then pose
an atemporal query in the situation which results from executing
$watch(t)$.  Similarly, one can introduce an auxiliary exogenous
action $waitFor(\phi)$ that is executed at a moment of time when the
condition $\phi$ becomes true,
where $\phi$ is composed from functional fluents that are interpreted as 
continuous functions of time. 
This approach has proved to be quite successful in  
cognitive robotics \cite{GrosskreutzLakemeyer2003}. For example, it has been used to provide a SC semantics \cbstart for \cbend continuous time variants of the popular planning language PDDL \cite{ClassenHuLakemeyer2007}.

\cbstart
In this paper we study  a {\it new} variant of temporal SC in which we can {\it directly} query continuously changing quantities at arbitrary points in time without introducing any actions
(either natural, or exogenous, or auxiliary) that supply the moment of time. Our approach works in a query-independent way. 
\cbend
For doing so we  take inspiration from the work on Hybrid Systems in 
Control Theory \cite{davoren,Nerode2007},
which are based on discrete transitions between states that continuously evolve over time. 

Following this idea, the crux of our proposal is to add a new kind of axioms called \emph{state evolution axioms} (SEA) to 
Reiter's \emph{successor state axioms} (SSA). The successor state axioms specify, as usual,
how fluents change when actions are executed. Informally, they characterize 
transitions between different states due to actions.
The state evolution axioms specify  how  the flow of time can bring  changes 
in system parameters within a given situation while no actions are executed.
Thus, we maintain the fundamental assumption of SC that all \emph{discrete} change is due to actions, though situations now include a temporal evolution.

Reiter \cite{reiter1991,reiter} shows how the SSA can be derived in atemporal SC from 
the effect axioms in normal form by making the causal completeness assumption. 
We do similar work wrt state evolution axioms, thus providing 
 a precise methodology for axiomatization of continuous processes in SC in
the spirit of hybrid systems.

\cbstart
One of the key results of SC is the ability to reduce reasoning about a future situation to reasoning about the initial state by means of regression \cite{reiter}. Despite that we now have continuous evolution in a situation, we show that a suitable notion of regression can still be defined. 
\cbend

We finally observe that, in hybrid automata, while continuous change
is dealt with thoroughly, the discrete description is typically
limited to finite state machines, i.e., it is based on a propositional
representation of the state. SC, instead, is based on a relational
representation of the state.
There are practical examples that call for such an extension of hybrid systems where states have an internal relational structure and the continuous flow of time determines the evolution within the state \cite{traffic}.
Our proposal can readily capture these cases, by providing a relational
extension to hybrid automata, which benefits from the representational
richness of SC.



In summary, wrt this point,
our work may serve as 
the spark that will bring together and cross-fertilize KR and Hybrid
Control, getting from the former the semantic richness of relational
states and from the latter a convenient treatment of continuous time.


%
The next section provides the technical background on SC and hybrid systems. Section \ref{s:main} presents our contributions to the temporal SC, including a derivation of state evolution axioms and a definition of temporal basic action theories. Section~\ref{s:example} illustrates the proposal on a full fledged example.   Section~\ref{s:regression} studies regression.  Section~\ref{s:related-sc} compares with previous proposals in AI and SC in particular.
Section \ref{s:ha} shows formally how hybrid automata can be captured by our temporal SC. Finally, Section \ref{s:conclusion} concludes the paper by discussing future work.

\section{Background}\label{s:bg}

\subsection{Situation Calculus}
Situation calculus (SC) is a second-order (SO) language for representing dynamic worlds. It has three basic sorts (situation, action, object) and a rich alphabet for constructing formulas over terms of these sorts. Reiter (\citeyear{reiter}) shows that to solve many reasoning problems about actions, it is convenient to work with SC \emph{basic action theories} (BATs) whose main ingredients are precondition axioms and successor state axioms. For each action
function $A(\bar{x})$, an \emph{action precondition axiom} (APA) has the syntactic
form
\begin{align*}
Poss(A(\bar{x}),s) \liff \Pi_A(\bar{x},s),
\end{align*}
meaning that the action $A(\bar{x})$ is possible in situation $s$ if and only if $\Pi_A(\bar{x},s)$ holds in $s$, where $\Pi_A(\bar{x},s)$ is a formula with free variables among $\bar{x} \eq (x_1, \ldots, x_n)$ and $s$. Situations are first order (FO) terms which denote possible world histories. A distinguished constant $S_0$ is used to denote the \emph{initial situation}, and function $do(\alpha, \sigma)$ denotes the situation that results from performing action $\alpha$ in situation $\sigma$. Every situation corresponds uniquely to a sequence of actions. We use $do([\alpha_1,\ldots,\alpha_n],S_0)$ to denote complex situation terms obtained by consecutively performing $\alpha_1,\ldots,\alpha_n$ in $S_0$. The notation $\sigma^\prime \sqsubseteq \sigma$ means that either situation $\sigma^\prime$ is a subsequence of situation $\sigma$ or $\sigma \eq \sigma^\prime$.
The formula $\forall a \forall s^\prime(do(a,s^\prime)\sqsubseteq \sigma \to Poss(a,s^\prime))$, abbreviated as $executable(\sigma)$, captures situations $\sigma$ all of whose actions are consecutively possible. Every BAT contains a set $\Sigma$ of domain-independent axioms which characterize situations as a single finitely branching infinite tree starting from $S_0$ such that, at each node $S$, each branch corresponds to a new situation $do(A, S)$ arising from execution of $A$, one of the finitely many actions, at $S$. Objects are FO terms other than actions and situations that depend on the domain of application. Above, $\Pi_A(\bar{x},s)$ is a formula \emph{uniform} in situation argument $s$: it does not mention the predicates $Poss$, $\sqsubseteq$, it does not quantify over variables of sort situation, it does not mention equality on situations, and it has no occurrences of situation terms other than the variable $s$ (see \cite{reiter}). 
For each relational fluent $F(\bar{x},s)$ and each functional fluent $f(\bar{x},s)$, respectively, a \emph{successor state axiom} (SSA) has the form
\begin{align*}
&F(\bar{x},do(a,s)) \liff \Phi_F(\bar{x},a,s),\\
&f(\bar{x},do(a,s))\eq y \liff \phi_f(\bar{x},y,a,s),
\end{align*}
where $\Phi_F(\bar{x},a,s)$ and $\phi_f(\bar{x},y,a,s)$ are formulas uniform in $s$, all of whose free variables are among those explicitly shown. 
(As usual, all free variables are $\forall$-quantified at front.)
In addition to $\Sigma$, the set $\mathcal{D}_{ap}$ of APAs, and the set $\mathcal{D}_{ss}$ of SSAs, a BAT $\mathcal{D}$ contains an \emph{initial theory}: a finite set $\mathcal{D}_{S_0}$ of FO formulas whose only situation term is $S_0$ (and possibly static facts without a situational argument). 
Finally, BATs
include a set $\mathcal{D}_{una}$ of unique name axioms for actions (UNA) specifying that two actions are different if their names are different and that identical actions have identical arguments. If a BAT has functional fluents, it is required to satisfy a consistency property whereby, for the right-hand side $\phi_f(\bar{x},y,a,s)$ of the SSA of each functional fluent $f$, there must exist a unique $y$ such that $\phi_f(\bar{x},y,a,s)$ is entailed by $\mathcal{D}_{S_0} \cup \mathcal{D}_{una}$.

BATs enjoy the \emph{relative satisfiability} property: a BAT $\mathcal{D}$ is satisfiable whenever $\mathcal{D}_{una} \cup \mathcal{D}_{S_0}$ is (see Theorem 4.4.6 in \cite{reiter}). This property allows one to disregard the more problematic parts of a BAT, like the second-order induction axiom, when checking satisfiability. Moreover, BATs benefit from \emph{regression}, a natural and powerful reasoning mechanism, invaluable for answering queries about the future (a problem known as \emph{projection}). The regression operator $\mathcal{R}$ is defined for sufficiently specific (\emph{regressable}) queries about the future, i.e. formulas without $\sqsubseteq$ or equality on situations where each situation term has the syntactic form $do([\alpha_1,\ldots,\alpha_n],S_0)$ and the action argument of each $Poss$ atom is bound to an action function. $\mathcal{R}[\varphi]$ is obtained from a formula $\varphi$ by recursively replacing each $Poss(A(\bar{t}),s)$ atom by $\mathcal{R}[\Pi_A(\bar{t},s)]$, each relational fluent atom $F(\bar{t},do(\alpha,\sigma))$ by $\mathcal{R}[\Phi(\bar{t},a,s)]$, and, for each functional fluent term $f(\bar{t},do(\alpha,\sigma))$ whose nested terms of sort $\textsc{object}$ or $\textsc{action}$ are uniform in $S_0$, replacing $\mathcal{R}[\varphi]$ by $\mathcal{R}[(\exists y).\phi_f(\bar{t},t,\alpha,\sigma) \land \varphi|^{f(\bar{t},do(\alpha,\sigma))}_y]$, where $A|^B_C$ denotes the substitution of $C$ for $B$ in $A$. A seminal result (Theorem 4.5.5, also p.73 in \cite{reiter}) has it that $\mathcal{D} \models \varphi$ iff $\mathcal{D}_{una} \cup \mathcal{D}_{S_0} \models \mathcal{R}[\varphi]$, meaning that regression reduces SO entailment to FO entailment by compiling dynamic aspects of the theory into the query.


To accommodate time, Reiter adds a temporal argument to all action functions and introduces two special function symbols. The symbol $time: \textsc{action} \mapsto \textsc{time}$ is used to access the time of occurrence of an action via its term and is specified by an axiom $time(A(\bar{x},t)) \eq t$ (included in $\mathcal{D}_{S_0}$) for every action function $A(\bar{x},t)$ in the alphabet of the BAT. The symbol $start : \textsc{situation} \mapsto \textsc{time}$ is used to access the starting time of situation $s$ and is specified by the new foundational axiom $start(do(a,s)) = time(a)$. The starting time of $S_0$ is not enforced, and the time points constituting the timeline with dense linear order are assumed to always have the standard interpretation (along with $+$, $<$, etc.). To outlaw temporal paradoxes, the abbreviation $executable(s)$ is redefined as
\begin{align*}
\forall a, s^\prime. do(a,s^\prime)\!\sqsubseteq\! s \to (Poss(a,s^\prime) \land start(s^\prime)\!\leq\!time(a)).
\end{align*}

Another useful notion is that of \emph{natural actions} --- non-agent actions that occur spontaneously as soon as their precondition is satisfied. Such actions are marked using the predicate symbol $natural$ as a part of $\mathcal{D}_{S_0}$, e.g., $natural(bounce(ball_1, t))$, and their semantics are encoded by a further modification of $executable(s)$. We use natural actions to induce relational change based on the values of the continuous quantities.

\subsection{Hybrid Systems}\label{ss:ha}
Hybrid automata are mathematical models used ubiquitously in control theory for analyzing dynamic systems which exhibit both discrete and continuous dynamics. 
\cite{davoren} define a \emph{basic hybrid automaton} (HA) as a system $H$ consisting of:
\begin{itemize}\setlength\itemsep{0em}
\item[--] a finite set $Q$ of \emph{discrete states};
\item[--] a \emph{transition relation} $E \subseteq Q \times Q$;
\item[--] a \emph{continuous state space} $X \subseteq \mathbb{R}^n$;
\item[--] for each $q\in Q$, a \emph{flow function} $\varphi_q:X \times \mathbb{R} \mapsto X$ and a set $Inv_q \subseteq X$ called the \emph{domain of permitted evolution};
\item[--] for each $(q,q^\prime) \in E$, a \emph{reset relation} $R_{q,q^\prime} \subseteq X \times \mathcal{P}(X)$. 
\item[--] a set $Init \subseteq \cup_{q\in Q}(\{q\}\times Inv_q)$ of \emph{initial states}.
\end{itemize}
Like a discrete automaton, a HA has discrete states and a state transition graph, but within each discrete state its continuous state evolves according to a particular flow, e.g., it can be an (implicit) solution to a system of differential equations. 
The domain of permitted evolution delineates the boundaries which the continuous state $X$ of the automaton cannot cross while in state $q$, i.e., $\varphi_q(X, t) \in Inv_q$. The reset relation helps to model discontinuous jumps in the value of the continuous state which accompany discrete state switching.

A \emph{trajectory} of a hybrid automaton $H$ is a 
sequence $\eta = \langle \Delta_i, q_i, \nu_i\rangle_{i\in I}$, with $I=\{1,2,\ldots\}$, such that for each $i \in I$:
\begin{enumerate}[label=(\alph*),align=left,leftmargin=*] \setlength\itemsep{0em}
\item the duration $\Delta_i \in \mathbb{R}^+ \cup \{\infty\}$, with $\Delta_i\eq \infty$ only if $I$ is finite and \ $i\eq |I|$, the number of elements in $I$;
\item $q_i \in Q$;
\item $\nu_i:[0,\Delta_i] \mapsto X$ is a continuous curve along the flow $\varphi_{q_i}$ that lies entirely inside $Inv_{q_i}$;
\item $(q_1,\nu_1(0)) \in Init$;
\item if $i < |I|$, then $(q_i, q_{i+1}) \in E$ and $(\nu_i(\Delta_i),$ $\nu_{i+1}(0))$ $\in$ $R_{q_i,q_i+1}$.
\end{enumerate}
A trajectory captures an instance of a legal evolution of a hybrid automaton over time. Duration $\Delta_i$ is simply the time spent by the automaton in the $i$-th discrete state it reaches while legally travelling through the transition graph, obeying the reset relation. The duration can be infinite if the automaton remains in the final discrete state $q_i$ indefinitely and the evolution of the continuous state within $q_i$, as described by the function $\nu_i$, never leaves the allowed domain $Inv_{q_i}$. A trajectory is \emph{finite} if it contains a finite number $|I|$ of steps and the sum of all durations, $\Sigma_{i\in I}\Delta_i$, is finite.

\section{Hybrid Temporal Situation Calculus}\label{s:main}
In our quest for a hybrid temporal SC, we reuse the temporal machinery introduced into BATs by Reiter, namely: all action symbols have a temporal argument and the functions $time$ and $start$ are axiomatized as described above. We preserve atemporal fluents, but no longer use them to model continuously varying physical quantities. Rather, atemporal fluents serve to specify the context in which continuous processes operate. For example, the fluent $Falling(b,s)$ holds if a ball $b$ is in the process of falling in situation $s$, indicating that, for the duration of $s$, the position of the ball (and its derivatives) should be changing as a function of time according to the equations of free fall. The fluent $Falling(b,s)$ may be directly affected by instantaneous actions $drop(b,t)$ (ball begins to fall at the moment of time $t$) and $catch(b,t)$ (ball stops at $t$), but the effect of these actions on the position of the ball comes about only indirectly, by changing the context of a continuous trajectory and thus switching the continuous trajectory that the ball can follow. More specifically, a falling ball is one context, and a ball at rest is another. 

In a general case, there are finitely many (parametrized) context types which are pairwise mutually exclusive when their parameters are appropriately fixed, and each context type is characterized by its own continuous function that determines how a physical quantity changes.

To model continuously varying physical quantities, we introduce new functional fluents with a temporal argument. We imagine that these fluents can change with time, 
and not only as a direct effect of instantaneous actions. 
For example, for the context where the ball is falling, the velocity of the ball 
at time $t$ represented by fluent $vel(b,t,s)$ can be specified as
\begin{align*}
&[Falling(b,s) \land y = vel(b,start(s),s) \!-\! g \!\cdot\!(t \!-\! start(s))]\\\
&\qquad \to  vel(b, t, s) \eq y.
\end{align*}
Notice that this effect axiom does not mention actions and describes the evolution of $vel$ within a single situation.


Formally, we augment SC with appropriate sorts to represent real-valued time, real-valued physical quantities, and accordingly extend the sets of predicate and function symbols of the language.

\subsection{Deriving State Evolution Axioms}

Our starting point is a \emph{temporal change axiom} (TCA) which describes a single law governing the evolution of a particular temporal fluent due to the passage of time in a particular context of an arbitrary situation. An example of a TCA was given above for $vel(b,t,s)$. We assume that a TCA for a temporal functional fluent $f$ has the general syntactic form
\begin{align}\label{eq:tca}
\gamma(\bar{x}, s) \land \delta(\bar{x},y,t,s) \to f(\bar{x}, t, s) \eq y,
\end{align}
where $t$, $s$, $\bar{x}$, $y$ are variables and $\gamma(\bar{x},s)$, $\delta(\bar{x},y,t,s)$ are formulas uniform in $s$ whose free variables are among those explicitly shown. We call $\gamma(\bar{x},s)$ the \emph{context} as it specifies the condition under which the formula $\delta(\bar{x},y,t,s)$ is to be used to compute the value of fluent $f$ at time $t$. Note that contexts are time-independent. The formula $\delta(\bar{x},y,t,s)$ may
define $y$ implicitly or explicitly using arbitrary computable domain-specific constraints (algebraic, differential, logical) on variables and fluents. This formula $\delta$ may 
encode differential equations, and if they do not have exact closed-form 
analytic solution, then value of $y$ is to be computed numerically.


A set of $k$ temporal change axioms for some fluent $f$ can be equivalently expressed as an axiom of the form
\begin{align}\label{eq:nf}
\Phi(\bar{x},y,t,s) \to f(\bar{x}, t, s) \eq y,
\end{align}
where $\Phi(\bar{x},y,t,s)$ is $\bigvee_{1\leq i \leq k}(\gamma_i(\bar{x}, s) \land  \delta_i(\bar{x},y,t,s))$. We additionally require that the background theory entails
\begin{align}\label{eq:unique}
\Phi(\bar{x},y,t,s) \land \Phi(\bar{x},y^\prime,t,s) \to y\eq y^\prime.
\end{align}
Condition (\ref{eq:unique}) guarantees the consistency of the axiom (\ref{eq:nf}) by preventing a continuous quantity from having more than one value at any moment of time. 
With condition (\ref{eq:unique}), we can assume  w.l.o.g. that all contexts in the given set of TCA are pairwise mutually exclusive wrt the background theory $\mathcal{D}$.

Having combined all laws which govern the evolution of $f$ with time into a single axiom (\ref{eq:nf}), we can make a causal completeness assumption: \emph{there are no other conditions under which the value of $f$ can change in $s$ from its initial value at $start(s)$ as a function of $t$}. We capture this assumption formally by the explanation closure axiom
\begin{align}\label{eq:eca}
\begin{split}
&f(\bar{x}, t, s) \neq f(\bar{x},start(s), s) \to \exists y \, \Phi(\bar{x},y,t,s).
\end{split}
\end{align}


\begin{theorem}\label{th:main}
Let for each formula of the form (\ref{eq:tca}) the background theory $\mathcal{D}$ entail $\forall(\gamma(\bar{x}, s) \to \exists y\, \delta(\bar{x},y,t,s))$. Then the 
conjunction of 
axioms (\ref{eq:nf}), (\ref{eq:eca}) in the models of (\ref{eq:unique}) are logically equivalent to
\begin{align}\label{eq:sea2}
\begin{split}
&
\hspace*{-2mm}
f(\bar{x},t,s)\eq y \liff [\Phi(\bar{x},y,t,s) \lor{}\\
&\quad \quad \quad \quad \quad \quad 
	y \eq f(\bar{x},start(s), s) \land \neg \Psi(\bar{x},y,t,s)],
\end{split}
\end{align}
where $\Psi(\bar{x},s)$ denotes $\bigvee_{1\leq i \leq k}\gamma_i(\bar{x}, s)$.
\end{theorem}
\begin{proof}
Full proof is provided in the Appendix.
\end{proof}

We call the formula (\ref{eq:sea2}) a \emph{state evolution axiom} (SEA) for the fluent $f$. 
Note what the SEA says: $f$ evolves with time during $s$ according to some law whose context is realized in $s$ or stays constant if no context is realized. 
Our causal completeness assumption (\ref{eq:eca}) applies both to physical quantities 
and  to their derivatives. This assumption simply states that all reasons for
change have been already accounted for in  (\ref{eq:nf}), nothing is missed.
It is important to realize that $\mathcal{D}_{se}$, a set of SEAs, 
complements the SSAs that are derived in \cite{reiter1991} using similar technique.

\subsection{Temporal Basic Action Theories}

The SEA for some temporal fluent $f$ does not completely specify the behaviour of $f$ because it talks only about change within a single situation $s$. To complete the picture, we need a SSA describing how the value of $f$ changes (or does not change) when an action is performed. A straightforward way to 
accomplish this would be by an axiom which would enforce continuity, e.g., $f(\bar{x},time(a),do(a,s)) \eq f(\bar{x},time(a),s)$. 
However, this choice would preclude the ability to model action-induced discontinuous jumps in the value of the continuously varying quantities or their derivatives, such as the sudden change of acceleration from $0$ to $-9.8 m/s^2$ when an object is dropped. To circumvent this limitation, for each temporal functional fluent $f(\bar{x}, t, s)$, we introduce an auxiliary atemporal functional fluent $f\init(\bar{x},s)$ whose value in $s$ represents the value of the physical quantity modelled by $f$ in $s$ at the time instant $start(s)$. We axiomatize $f\init$ using a SSA derived from the axioms
\cbstart
\begin{align*}
&e(\bar{x},y,a,s) \to f\init(\bar{x}, do(a,s)) = y,\\\
&\neg\exists y(e(\bar{x},y,a,s)) \to f\init(\bar{x}, do(a,s)) \eq f(\bar{x},time(a),s),
\end{align*}
where the former is a Reiter's effect axiom in normal form and the latter asserts that if no relevant effect is invoked by the action $a$, $f\init$ assumes the most recent value of the continuously evolving fluent $f$. This latter axiom enforces temporal continuity in the value of $f$ in the case when there is no reason for change.
\cbend

 The SSA for $f\init$ has the general syntactic form
\begin{align}\label{eq:init}
\begin{split}
&f\init(\bar{x}, do(a,s)) \eq y \liff \Omega(\bar{x}, y, a, s),
\end{split}
\end{align}
where $\Omega(\bar{x},y, a, s)$ is a formula uniform in $s$ 
whose purpose is to describe how the initial value of $f$ in $do(a,s)$ relates to its value at the same time instant in $s$ (i.e., prior to $a$). To establish a consistent relationship between temporal fluents and their atemporal \emph{init}-counterparts, we require that, in an arbitrary situation, the continuous evolution of each temporal fluent $f$ starts with the value computed for $f\init$ by its successor state axiom.


A \emph{temporal basic action theory} is a collection of axioms $\mathcal{D} = \Sigma \cup \mathcal{D}_{ss} \cup \mathcal{D}_{ap} \cup \mathcal{D}_{una} \cup \mathcal{D}_{S_0} \cup \mathcal{D}_{se}$ such that
\begin{enumerate}\setlength\itemsep{0.1em}
\item Every action symbol mentioned in $\mathcal{D}$ is temporal;
\item $\Sigma \cup \mathcal{D}_{ss} \cup \mathcal{D}_{ap} \cup \mathcal{D}_{una} \cup \mathcal{D}_{S_0}$ constitutes a BAT as per  Definition 4.4.5 in \cite{reiter};
\cbstart
\item $\mathcal{D}_{se}$ is a set of state evolution axioms of the form
\begin{align}\label{eq:seagf}
f(\bar{x},t,s) \eq y \liff \psi_f(\bar{x},t,y,s)
\end{align}
where $\psi_f(\bar{x},t,y,s)$ is uniform in $s$, such that for every temporal functional fluent $f$ in $\mathcal{D}_{se}$, $\mathcal{D}_{ss}$ contains an additional SSA of the form (\ref{eq:init}) for $f\init$;
\item For each SEA of the form (\ref{eq:seagf}), the following consistency properties hold:
\begin{align}\label{eq:func-consistency}
\begin{split}
&\mathcal{D}_{una} \cup \mathcal{D}_{S_0} \models{}\\ 
&\quad \forall \bar{x} \forall t. \;\exists y (\psi_f(\bar{x}, t, y, s)) \land{} \\
&\quad \forall y \forall y^\prime (\psi_f(\bar{x}, t, y, s) \land \psi_f(\bar{x}, t, y^\prime, s) \to y \eq y^\prime),
\end{split}
\end{align}
\begin{align}\label{eq:init-consistency}
\begin{split}
&\mathcal{D}_{una} \cup \mathcal{D}_{S_0} \models{} \\ 
&\quad \exists y ( f\init(\bar{x},s) \eq y \land{} \psi_f(\bar{x},start(s), y, s)); 
\end{split}
\end{align}
\end{enumerate}
%
A set $\mathcal{D}_{se}$ of SEA (\ref{eq:seagf}) is \emph{stratified} iff there are no temporal fluents $f_1, \ldots, f_n$ such that $f_1 \succ f_2 \succ \ldots \succ f_n \succ f_1$ where $f \succ f'$ holds iff there is a SEA in $\mathcal{D}_{se}$ where $f$ appears on the left-hand side and $f'$ on the right-hand side. A temporal BAT is \emph{stratified} iff its $\mathcal{D}_{se}$ is.

Similarly to Reiter's BATs, temporal BATs have a relative satisfiability property.
%
\begin{theorem}
A stratified temporal BAT $\mathcal{D}$ is satisfiable iff $\mathcal{D}_{una} \cup \mathcal{D}_{S_0}$ is satisfiable. 
\end{theorem}
\begin{proof}[Proof (sketch)]
The proof extends that of Theorem 1 in \cite{pirri99}. We start with a model $\mathcal{M}_0$  of $\mathcal{D}_{una} \cup \mathcal{D}_{S_0} $ and show how to build a model $\mathcal{M}$ of $\mathcal{D}$. Let $\mathcal{M}$ be constructed from $\mathcal{M}_0$ as in \citeauthor{pirri99} with the additional condition that for each temporal fluent $f$, $f^\mathcal{M}(\bar{x}[v_0], t[v_0], S_0^\mathcal{M}) = f^{\mathcal{M}_0}(\bar{x}[v_0], t[v_0], S_0^{\mathcal{M}_0})$, so that $\mathcal{D}$ is satisfied by $\mathcal{M}$ at $S_0$.
Assume that $\mathcal{M}$ interprets all symbols except temporal fluents at an arbitrary situation $s$. For each temporal fluent $f$ with SEA of the form (\ref{eq:seagf}), let $\mathcal{M},v \models f(\bar{x},t,s) \eq y$ iff $\mathcal{M},v \models \psi_f(\bar{x},t,y,s)$. This is well-defined because the right-hand side of the SEA for the lowest-stratum temporal fluent does not mention other temporal fluents and therefore has already been assigned a truth value. The remaining temporal fluents are assigned values by induction on strata. By (\ref{eq:func-consistency}), these values are unique and, by (\ref{eq:init-consistency}), consistent with the discrete dynamics of temporal fluents.
\end{proof}
\cbend

\section{An Example}\label{s:example}
Consider a macroscopic urban traffic domain along the lines of \cite{traffic}. For simplicity, we consider a single intersection of two 2-lane roads. Facing the intersection $i$ are 4 incoming and 4 outgoing road segments. Depending on the traffic light, a car may turn left, turn right, or drive straight from an incoming lane to the corresponding outgoing lane. The layout of the intersection $i$ is shown on Figure \ref{fig:intersection} (left). Each lane is denoted by a constant and each path through the intersection $i$ is encoded using the static predicates $st(i, r_1, r_2)$ (straight connection from lane $r_1$ to $r_2$ at intersection $i$), $lt(i, r_1, r_2)$ (left turn), and $rt(i, r_1, r_2)$ (right turn). The number of cars per unit of time that can pass through each connection is specified by the function $flow(i, r_1, r_2)$. 

The outgoing lanes are assumed to be of infinite capacity and are not modelled. The traffic lights are controlled by a simple looping automaton with the states $Green(i,r,s)$ (from lane $r$, go straight or turn right), followed by $RArr(i,r,s)$ (\emph{right arrow}, i.e., only turn right), followed by $Red(i,r,s)$ (stop), and then $LArr(i,r,s)$ (\emph{left arrow}, i.e., only turn left), such that mutually orthogonal directions are in antiphase to each other. The switching between these states for all $r$ is triggered by the action $switch(i,t)$ with precondition $Poss(switch(i, t),s) \liff start(s) \lleq t$ via a set of simple SSA, e.g.,
\begin{align*}
& Green(i, r, do(a,s)) \liff a \eq switch(i, t) \land LArr(i, r,s) \lor{}\\
&\quad Green(i, r,s) \land a \nneq switch(i, t).
\end{align*}
These SSA ensure the correct order of signals as shown on Figure \ref{fig:intersection} (right). To rule out multiple signals and maintain the correct correspondence between the signals for intersecting directions, we require that the initial state axioms entail a set of simple state constraints.

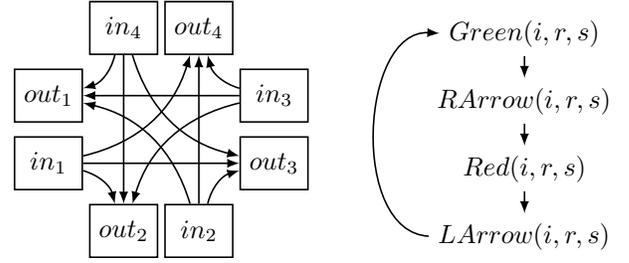
\begin{figure}[h]
\centering
\subfloat{
\begin {tikzpicture}[-latex, auto, node distance=0.9 cm and 1cm, on grid, semithick, state/.style ={draw,minimum width=0.9 cm, minimum height=0.7cm,}]
\node[state,] (a){${out_1}$};
\node[state] (d) [below=of a] {$in_1$};
\node[state,] (e) [below right=of d] {${out_2}$};
\node[state] (h) [right =of e] {$in_2$};
\node[state,] (i) [above right=of h] {${out_3}$};
\node[state] (l)  [above=of i] {$in_3$};
\node[state,] (m) [above left=of l] {${out_4}$};
\node[state] (p)  [left=of m] {$in_4$};


\path[]  (p) edge [ bend left=30] (a);
\path[]  (p) edge [ bend left=0] (e);
\path[]  (p) edge [ bend right=30] (i);

\path[]  (d) edge [ bend left=30] (e);
\path[]  (d) edge [ bend left=0] (i);
\path[]  (d) edge [ bend right=30] (m);

\path[]  (h) edge [ bend left=30] (i);
\path[]  (h) edge [ bend left=0] (m);
\path[]  (h) edge [ bend right=30] (a);

\path[]  (l) edge [ bend left=30] (m);
\path[]  (l) edge [ bend left=0] (a);
\path[]  (l) edge [ bend right=30] (e);
\end{tikzpicture}
}\quad
\subfloat{
\begin {tikzpicture}[-latex, auto, node distance=0.9 cm and 1cm, on grid, semithick, state/.style ={minimum width=0.1 cm, minimum height=0.1cm,}]
\node[state] (green){$Green(i,r,s)$};
\node[state] (rarrow) [below =of green] {$RArrow(i,r,s)$};
\node[state] (red)[below =of rarrow]{$Red(i,r,s)$};
\node[state] (larrow)[below =of red]{$LArrow(i,r,s)$};

\path[]  (green) edge [ bend left=0] (rarrow);
\path[]  (rarrow) edge [ bend left=0] (red);
\path[]  (red) edge [ bend left=0] (larrow);
\path[]  (larrow) edge [ bend left=90] (green);
\end{tikzpicture}
}
\caption{Intersection layout and the states of the traffic lights}\label{fig:intersection}
\end{figure}
 
The continuous quantity that we wish to model is the number of cars queued up in each incoming lane. For that, we use the temporal functional fluent $que(i,r,t,s)$ and its atemporal counterpart $que\init(i,r,s)$ --- the number of cars at intersection $i$ waiting in lane $r$ in situation $s$ at times $t$ and $start(s)$, respectively. Since the lane $r$ may run dry and thus affect the continuous dynamics of queues, we call on the natural action $empty(i,r,t)$ to change the relational state when that happens:
\begin{align*}
&Poss(empty(i, r, t), s) \liff start(s) \lleq t \land que(i, r, t, s) \eq 0,\\
& a \eq empty(i, r,t) \land y \eq 0 \to que\init(i, r,do(a,s)) \eq y,\\
& a \nneq empty(i, r,t) \land y \eq que(i, r,time(a),s) \\
&\qquad \to que\init(i, r,do(a,s)) \eq y.
\end{align*}

The left-hand sides of the above effect axioms for $que\init$ are mutually exclusive, so the right-hand side of the SSA for $que\init(i,r,do(a,s))$ is simply a disjunction thereof:
\begin{align*}
&que\init(i, r,do(a,s)) \eq y \liff  [a \eq empty(i, r,t) \land y \eq 0] \lor{}\\
&\quad [a \nneq empty(i, r,t) \land y \eq que(i, r,time(a),s)].
\end{align*}

We can now formulate the TCA for $que$ 
according to traffic rules, i.e., cars move when they are allowed to. Cars do not move at a red light, resulting in the simple TCA
\begin{align*}
\big[Red(i,r,s) \land  y \eq que\init(i,r, s) \big] \to que(i, r, t, s) \eq y.
\end{align*}

When the signal for a non-empty lane $r$ is the left arrow, the queue of $r$ decreases linearly from its initial size with the rate associated with the left turn:
\begin{align*}
&\big[LArr(i,r,s)\land que\init(i,r,s) \nneq 0 \land \exists r'  [lt(i,r, r') \land{}\\
&\quad y \eq ( que\init(i,r,s)\!-\!flow(i,r,r')\!\cdot\!(t\!-\!start(s))]\big]\\
&\qquad \to que(i,r, t, s) \eq y.
\end{align*}

The queue of a non-empty lane with signal $RArr(i,r,s)$ decreases similarly. For the signal $Green(i,r,s)$, the queue decreases with a combined rate of the straight connection and the right turn:
\begin{align*}
&\big[Green(i,r,s)\land que\init(i,r,s) \nneq 0 \land{}\\
&\; \exists r'  \exists r''[st(i,r, r') \land rt(i,r,r'') \land y \eq ( que\init(i,r,s) -\\
&\qquad (flow(i,r,r')\!+\!flow(i,r,r''))\!\cdot\!(t\!-\!start(s))]\big]\\
&\qquad\qquad \to que(i,r, t, s) \eq y.
\end{align*}

From these TCA, by Theorem \ref{th:main}, we obtain a SEA (simplified for brevity)
\begin{align*}
&que(i,r,t,s) \eq y \liff{} (\exists \tau \exists q_0 \exists r_L \exists r_S \exists r_R).\\
&\quad  \tau \eq (t\!-\!start(s)) \land q_0 \eq que\init(i,r, s) \land{}\\
&\quad  lt(i,r, r_L) \land st(i,r, r_S) \land rt(i,r, r_R) \land{}\\
&\quad \big[LArr(i,r,s) \land q_0 \nneq 0 \land y \eq (q_0\!-\!flow(i,r,r_L) \cdot \tau) \lor{}\\
&\quad \phantom{\big[}Green(i,r,s) \land q_0 \nneq 0 \land{}\\
&\qquad\qquad  y \eq  (q_0\!-\!(flow(i,r,r_S)\!+\!flow(i,r,r_R)) \cdot \tau) \lor{}\\
&\quad \phantom{\big[}RArr(i,r,s) \land q_0 \nneq 0 \land y \eq  (q_0\!-\!flow(i,r,r_R) \cdot \tau) \lor{}\\
&\quad \phantom{\big[}Red(i,r,s) \land y \eq q_0 \lor{}\\
&\quad \phantom{\big[}\neg Red(i,r,s) \land q_0 \eq 0 \land y \eq 0\big]. 
\end{align*}
Notice that the last line comes not from the TCA but from the 
explanation closure (\ref{eq:eca})
enforced by Theorem \ref{th:main} and asserts the constancy of $que$ in the context which the TCA did not cover (movement is allowed but the lane is empty).

In general, the modeller only needs to supply the TCA for the contexts where the quantity changes with time.

\section{Regression}\label{s:regression}

Projection is a ubiquitous computational problem concerned with establishing the truth value of a statement after 
executing a given sequence of actions.
We solve it with the help of regression. It turns out that the notions of uniform and regressable formulas trivially extend to temporal BATs. In this section, for the sake of simplicity and practicality, without loss of generality, we assume that each regressable formula can mention $do$ only in a situation term or within the function $start$ in temporal arguments. The regression operator $\mathcal{R}$ as defined for atemporal BATs in Definition 4.7.4 of \cite{reiter} can be extended to temporal BATs in a straightforward way. When $\mathcal{R}$ is applied to a regressable formula $W$, $\mathcal{R}[W]$ is determined relative to a temporal BAT. We extend $\mathcal{R}$ as follows.

Let $\mathcal{D}$ be a temporal BAT, and let $W$ be a regressable formula.
If $W$ is a non-fluent atom that mentions $start(do(\alpha,\sigma))$, then $\mathcal{R}[W] = \mathcal{R}[W|_{time(\alpha)}^{start(do(\alpha,\sigma))}]$.
If $W$ is a non-$Poss$ atom and mentions a functional fluent uniform in $\sigma$, then this term is either atemporal or temporal. The former case is covered by Reiter. In the latter case, the term is of the form $f(\bar{C}, \tau^\star, \sigma)$ and has a SEA $f(\bar{x},t,s) \eq y \liff \psi_f(\bar{x},t,y,s)$, so we rename all quantified variables in $\psi_f(\bar{x},t,y,s)$ to avoid conflicts with the free variables of $f(\bar{C},\tau^\star,\sigma)$ and define $\mathcal{R}[W]$ to be 
\begin{align*}
&\mathcal{R}[\exists y. \,(\tau^\star \eq start(\sigma) \land y \eq f\init(\bar{x},\sigma) \lor{}\\
&\hphantom{\mathcal{R}[}\tau^\star \nneq start(\sigma) \land \psi_f(\bar{C},\tau^\star,y,\sigma)) \land W|^{f(\bar{C},\tau^\star, \sigma)}_y],
\end{align*}
where $y$ is a new variable not occurring free in $W$, $\bar{C}$, $\tau^\star$, $\sigma$. Intuitively, this transformation replaces the temporal fluent $f$ with either the value of $f\init$ if $f$ is evaluated at the time of the last action or, otherwise, with the value determined by the right-hand side of the SEA for $f$.

\begin{theorem}
If $W$ is a regressable sentence of SC and $\mathcal{D}$ is a stratified temporal basic action theory, then
\begin{align*}
\mathcal{D} \models W \quad\text{iff}\quad \mathcal{D}_{S_0} \cup \mathcal{D}_{una} 
\models \mathcal{R}[W].
\end{align*}
\end{theorem}
\begin{proof}[Proof (sketch)]
The proof proceeds exactly like that of Theorem 2 in \cite{pirri99} with the following addition. Since the right-hand side of a SEA is uniform in $s$, we only need to show that all occurrences in $W$ of temporal fluents are eventually replaced with logically equivalent expressions which can be regressed by usual means.

As long as variables are properly renamed, we have
\begin{align*}
\mathcal{D} \models (\forall)(W \leftrightarrow \exists y(\psi_f(\bar{C},\tau^\star,y,\sigma) \land W|^{f(\bar{C},\tau^\star, \sigma)}_y)).
\end{align*}
By the property (\ref{eq:init-consistency}) of a temporal BAT, $\mathcal{D} \models \exists y. f_{init}(\bar{x},s) \eq y \land \psi_f(\bar{x},start(s),y,s)$, so $\mathcal{D}$ entails
\begin{align*}
&(\forall)(W \leftrightarrow \exists y. \,(\tau^\star \eq time(\alpha) \land y \eq f_{init}(\bar{x},do(\alpha,\sigma)) \lor{}\\
&\qquad\qquad \tau^\star \nneq time(\alpha) \land \psi_f(\bar{\tau},\tau^\star,y,\sigma)) \land W|^{f(\bar{\tau},\tau^\star, \sigma)}_y).
\end{align*}
The formula $W|^{f(\bar{C},\tau^\star, \sigma)}_y$ mentions exactly one temporal functional fluent fewer than $W$.
Since $\mathcal{D}_{se}$ is stratified, the formula $\psi_f(\bar{C},\tau^\star,y,\sigma))$ mentions temporal functional fluents of a strictly lower stratum than $f$.
Since $W$ is a finite, repeated application of temporal $\mathcal{R}$ to $W$ involves a finite number of steps and yields an expression $(\exists)(\phi \land W')$ where $W'$ mentions no temporal fluents and $\phi$ mentions only temporal fluents of a strictly lower stratum than the highest-stratum temporal functional fluent of $W$. (Repeated application of general $\mathcal{R}$ to $W$ may raise the maximum stratum of the expression due to temporal preconditions and context conditions of SSA, but only finitely many times). After a finite number of steps, regression arrives at an expression which mentions no temporal fluents. Our theorem follows by Theorem 3 of \cite{pirri99}.
\end{proof}
\commentout{	
There are several cases of the usual BATs when the latter entailment problem 
(without $\mathcal{D}_{se}$)  is decidable \cite{reiter,GuSoutchanski2010}. 
In all those cases, we can answer the projection queries in a temporal BAT if 
$\mathcal{D}_{se}$ is formulated in the first order  language of the real numbers 
since the theory of real closed fields is decidable by the Tarski-Seidenberg theorem.
}	

\begin{example}[continued] Consider the example from the previous section. Observe that the only SEA of the theory is stratified in that its right-hand side does not reference any temporal fluents.

Let the initial state be noncontradictory and entail the following facts about $S_0$:
\begin{align*}
&start(S_0) = 0, \;Red(I, in_1, S_0), \;que\init(I,in_1, S_0) = 100,\\
& flow(I,in_1, out_2) = 5, \;flow(I,in_1, out_3) = 15,\\
& flow(I,in_1, out_4) = 10.
\end{align*}
Let $W$, the statement of interest, be $que(I, in_1, 3, \sigma) < 95$, i.e., there are fewer than 95 cars in the incoming lane $in_1$ at time $3$ in situation $\sigma$, where $\sigma$ is $do([switch(I,1), switch(I,2)],S_0)$. In this narrative, the lane $in_1$ initially sees the red light, which at $t\eq1$ switches to the left arrow, and at $t\eq2$ to the green light. $W$ is clearly a regressable formula. To determine whether $W$ is entailed by the theory, we can use Theorem 3 to reduce $W$ to a logically equivalent statement about the initial situation whose entailment from $\mathcal{D}_{S_0} \cup \mathcal{D}_{una}$ can be computed by FO theorem proving. (For brevity, we perform the following simplifications on the formulas. First, we replace expressions $start(do(A(\bar{x},t),s)) = t'$ by $true$ ($false$) if $t=t'$ ($t \neq t'$).  Second, we replace relational fluents by their truth values if they can be established by the usual atemporal regression. Third, we exploit unique names for actions.)
\begin{align*}
& \mathcal{R}[que(I, in_1, 3, \sigma) < 95] =\\
& \mathcal{R}[\exists y \exists q_0(y\!<\!95 \land  q_0 \nneq 0 \land  y \eq  q_0\!-\!(15\!+\!5) (3-2) \land{}\\
&\quad q_0 \eq que\init(I,in_1,do([switch(I,1), switch(I,2)],S_0)))]=\\
& \mathcal{R}[\exists y \exists q_0(y\!<\!95 \land q_0 \nneq 0 \land  y \eq  q_0\!-\!(15\!+\!5) (3-2) \land{}\\
&\quad q_0 \eq que(I,in_1,2, do(switch(I,1),S_0)))] = \\
& \mathcal{R}[\exists y \exists q_0(y\!<\!95 \land q_0 \nneq 0 \land  y \eq  q_0\!-\!(15\!+\!5) (3-2) \land{}\\
&\quad \exists y' \exists q_0'(q_0 \eq y' \land q_0' \nneq 0 \land y' \eq q_0'\!-\!10  (2-1) \land{}\\
&\qquad q'_0 \eq que\init(I, in_1, do(switch(I,1),S_0))))]=\\
& \mathcal{R}[\exists y \exists q_0(y\!<\!95 \land q_0 \nneq 0 \land  y \eq  q_0\!-\!(15\!+\!5)  (3-2) \land{}\\
&\quad \exists y' \exists q_0'(q_0 \eq y' \land q_0' \nneq 0 \land y' \eq q_0'\!-\!10  (2-1) \land{}\\
&\qquad q'_0 \eq que(I, in_1, 1, S_0)))]=\\
&\exists y \exists q_0(y\!<\!95 \land q_0 \nneq 0 \land  y \eq  q_0\!-\!(15\!+\!5)  (3-2) \land{}\\
&\quad \exists y' \exists q_0'(q_0 \eq y' \land q_0' \nneq 0 \land y' \eq q_0'\!-\!10  (2-1) \land{}\\
&\qquad \exists y'' (q'_0 \eq y'' \land y''\eq que\init(I, in_1, S_0))))
\end{align*}
The end result further simplifies to 
\begin{align*}
que\init(I, in_1, S_0) - 10(2-1) - (15+5)(3-2) < 95,
\end{align*}
which is a query about $S_0$ and can be answered by plugging 100 for the initial number of cars at $in_1$: $100-10-20 = 70$, which is strictly less than 95, so the statement is true.

In addition to computing entailment, regression can be a powerful diagnostic tool. By analyzing the results of partial regression of a regressable temporal query, it is possible to attribute its validity to a particular action of the given sequence. Let $\mathcal{R}^{\sigma'}$ be a variant of $\mathcal{R}$ which does not regress beyond $\sigma'$, i.e., $\mathcal{R}^{\sigma'}[W]$ is uniform in $\sigma'$ if $\sigma' \sqsubseteq \sigma$ and $W$ is uniform in $\sigma$. We can establish whether $\mathcal{R}^{\sigma'}[W]$ is true for each $\sigma' \sqsubseteq do([switch(I,1), switch(I,2)], S_0)$ as just demonstrated (also using $\mathcal{R}$). In our example, the query holds continuously during and after the action $switch(I,2)$ but is false before and at the instant of the action $switch(I,1)$. We can conclude that the action $switch(I,1)$ as well as the time that has passed since $t\eq1$ and up to the time when $\mathcal{R}^{do(switch(I,1), S_0)}[W]$ became true are responsible for the fact that $W$ holds at $\sigma$. In other words, the fact that the traffic light changed from $Red$ to $LArr$ has allowed cars to flow, and the subsequent passage of time (which is easily computed to be equal to 0.5) caused the query to hold. Recall that the query can be an arbitrary regressable property of the continuous system.

\end{example}

\section{Comparison with Previous Approaches}\label{s:related-sc}

Considering that discrete-continuous systems have been a hot topic for decades, it is impossible to fairly compare hybrid situation calculus to a representative subset of all work in that area. Hence, we draw comparisons only to approaches from the same paradigm.

AI proposals to formalize and reason about physical systems with hybrid temporal behaviours go at least as far back as the work of \cite{dekleer} on na\"ive physics. \citeauthor{dekleer} model physical systems using \emph{confluences}---equations obtained from differential equations by reducing them to only qualitative relationships between quantities. \cite{SandewallKR89} points out that discarding information from a theory cannot lead to better inferences. He argues that differential calculus is the perfect language for modelling continuous change and that the essential task in describing physical systems is to provide a logical foundation for the discrete state transitions. 




The work in \cite{PintoPhD,pinto1995} presents initial proposals
to introduce time into the situation calculus; they focused on a so-called
actual sequence of actions and introduced representation for occurrences of 
actions wrt an external time-line. 
Chapter 6 of  \cite{PintoPhD} discusses several examples
of continuous change and natural events following \cite{SandewallKR89}, but 
without using Sandewall's non-monotonic solution to the frame problem. 
Also, it introduces a new class of objects called parameters
that are used to name continuously varying properties such that each parameter
behaves according to a unique function of time during a fixed situation. 
It is mentioned that parameters can be replaced with functional fluents of time,
but this direction was not elaborated. 
%
\cite{Miller1996} is another proposal to represent continuous change in SC.
Building on earlier work of \cite{SandewallKR89,PintoPhD},
\cite{Miller1996} introduces time-independent fluents and situation-independent 
parameters that can be regarded as functions of time, but 
provides only an example, and no general methodology.
%
%
\cite{ReiterKR96} provides the modern axiomatization of time, concurrency, and
natural actions in SC, and it appears also in \cite{reiter}. 
However, \cite{ReiterKR96} allows only atemporal fluents in
contrast to \cite{PintoPhD}. For this reason, \cite{mes99} proposes an auxiliary
action $watch(t)$ to monitor the execution of robot programs in real time.
\commentout{    
\cite{ClassenHuLakemeyer2007} provide a SC semantics for an expressive fragment of the Planning Domain Definition Language (PDDL) where they borrow the idea of \citeauthor{GrosskreutzLakemeyer2003}.
to 
map continuous fluents to functions of time. In their research, only linear functions of time are allowed.
}       


The example in Section~\ref{s:example} helps illustrate the differences with our approach.
Consider Reiter's temporal SC \cite{reiter}, 
which forms the backbone of ours. Since Reiter's fluents are atemporal, the TCA above are replaced by sets of effect axioms for the atemporal fluent $que(i,r,s)$, i.e.,
\begin{align*}
&a \eq switch(i,t) \land{}\\
&\big[LArr(i,r,s)\land que(i,r,s) \nneq 0 \land \exists r'  [lt(i,r, r') \land{}\\
&\quad y \eq ( que(i,r,s)\!-\!flow(i,r,r')\!\cdot\!(time(a)\!-\!start(s))]\big]\\
&\qquad \to que(i,r,do(a,s)) \eq y.
\end{align*}
Note that, in effect axioms, the change in $que$ is associated with a named action. The modeller must replicate this axiom for each action which might affect the context $LArr(i,r,s)\land que(i,r,s) \nneq 0$, and likewise for all other contexts and TCA. In our approach, the change in context is handled separately and does not complicate the axiomatization of continuous dynamics.
The right-hand side 
of the resulting SSA 
\begin{align*}
&que(i,r,do(a,s))\eq y \liff{}\\
&\quad \gamma_{que}(i,r,y,s)\lor que(i,r,s)\eq y \land \neg \exists y' \gamma_{que}(i,r,y',s)
\end{align*}
can be obtained from the right-hand side of the SEA above by replacing $t$ with $time(a)$, $que\init(i,r,s)$ with $que(i,r,s)$, and the last line by the expression $que(i,r,s) \eq y \land \neg \exists y' \gamma_{que}(i,r,y',s)$.
Notice that the expression $\gamma_{que}(i,r,y,s)$ occurs twice --- first thanks to 
the effect axiom (in a normal form) and then again due to explanation closure ---
see examples in Section 3.2.6 in \cite{reiter}. In our approach, by Theorem \ref{th:main}, only the essential atemporal part of that expression appears.
Furthermore, Reiter's version of the precondition axiom for $empty(i,r,t)$ 
is necessarily cumbersome 
because it mentions $que(i,r, t, s)$, whose evolution (and thus the value at $t$) depends on the current relational state of $s$. Therefore, the modeller must include 
the right-hand side of the SSA in the precondition, 
thereby increasing the size of the axioms by roughly the size of the SSA for the continuous fluent $F$ for every mention of $F$ in a precondition axiom while not adding any new information. Moreover, since Reiter's fluents are atemporal, evaluating them at an arbitrary moment of time $t$ requires an auxiliary action.

The approach due to \cite{mes99} introduces the special action $watch(t)$ whose purpose is to advance time to the time-point $t$. This mechanism allows one to access continuous fluents in between
the agent actions, 
but at a cost: replacing $que(i,r,t,s)$ by $que(i,r,do(watch(t),s))$ in the precondition axiom makes the right-hand side non-uniform in $s$, violates Defn. 4.4.3 in \cite{reiter}, and therefore steps outside of the well-studied realm of basic action theories.

A proposal due to \cite{GrosskreutzLakemeyer2003} considers continuous fluents of a different kind: their values range over functions of time, but neither the continuous fluents nor the action symbols have a temporal argument. Domain actions occur at the same time instant as the preceding situation, and the mechanism for advancing time is the special action $waitFor(\phi)$ which simulates the passage of time until the earliest time point where the condition $\phi$ holds, where $\phi$ is a Boolean combination of comparisons between values of continuous fluents and numerical constants. Aimed specifically at robotic control, this approach relies on a \textsf{cc-Golog} program to trigger the occurrence of the $waitFor$ action.

\cite{FinziPirriIJCAI05} introduce \emph{temporal flexible situation calculus}, a dialect aimed to provide formal semantics and a \textsf{Golog} implementation for constraint-based interval planning which requires dealing with multiple alternating timelines. To represent processes, they introduce fluents with a time argument. However, this time argument marks the instant of the process' creation and is not associated with a continuous evolution.

\section{Modeling Hybrid Automata}\label{s:ha}

Temporal BATs introduced here are naturally suitable for capturing hybrid automata \cite{Nerode2007}. Given an arbitrary basic hybrid automaton $H$, c.f., Section \ref{ss:ha}, we proceed as follows. For every discrete state in the finite set $Q$, we introduce a situation calculus constant $q_i$ with $1 \leq i \leq |Q|$ and let $\mathcal{D}_{S_0}$ contain unique name axioms for all $q_i$. We assume that the transition relation $E$ is encoded by a finite set of static facts $Edge(q,q^\prime)$. Each flow $\varphi_q$ is encoded by the situation-independent function $flow$ such that $flow(q,x,t) \eq y$ iff $\varphi_q(x,t) \eq y$. Each set of invariant states $Inv_q$ is encoded by the static predicate $Inv(q,x)$ which holds iff $x \in Inv_q$. Each reset relation $R_{q,q^\prime}$ is encoded by the static predicate $Reset(q, q^\prime, x, y)$ which holds iff $y \in R_{q,q^\prime}(x)$. The set of initial states $Init$ is encoded by the static predicate $Init(q,x)$ which holds iff $(q, x) \in Init$.


Let $trans(q,q^\prime,y,t)$ be the only action symbol representing a transition from a discrete state $q$ to a discrete state $q^\prime$ at time instant $t$ while resetting the continuous state to the value $y$. Let the atemporal functional fluent $Q(s)$ describe the discrete state in situation $s$, and let the temporal functional fluent $X(t,s)$ describe the continuous state in situation $s$ at time $t$. The dynamics of discrete state transitions and the evolution of the continuous state variables can be axiomatized as follows.
\begin{align*}
&Poss(trans(q,q^\prime,y,t),s) \liff Q(s)\eq q \land Edge(q,q^\prime) \land{}\\
&\quad  \exists x(X(t,s)\eq x \land Reset(q,q^\prime,x, y) \land Inv(q^\prime, y)),
\end{align*}
\begin{align*}
&Q(do(a,s))\eq q \liff \exists q^\prime \exists y \exists t (a\eq trans(q^\prime, q, y, t)) \lor{}\\
&\qquad Q(s) \eq q \land \neg \exists q^{\prime} \exists y \exists t (a\eq trans(q, q^{\prime}, y, t)),
\end{align*}
\begin{align*}
&X\init(do(a,s)) \eq x \liff \exists q \exists q^\prime \exists t (a\eq trans(q, q^\prime, x, t)),\\ 
&X(t,s) \eq x \liff \textstyle\bigvee_{i=1}^k [Q(s) \eq q_i \land x \eq flow(q_i, X\init(s), t).
\end{align*}
The precondition axiom for $trans$ states that a transition from $q$ to $q^\prime$ while resetting continuous state to $y$ is possible at time $t$ iff the current discrete state is $q$, there is an edge from $q$ to $q^\prime$ in the graph, the reset relation determines the new continuous state $y$, and the resulting state is legal. The SSA for the discrete state asserts that $q$ is the current state iff we transition into it and do not transition out. The initial value SSA for the continuous state forces $X\init$ to take on the value prescribed by the reset relation. Finally, the continuous state evolves in each situation $s$ starting with the initial value $X\init(s)$ according to the flow associated with the current discrete state. 
The 
proposed translation is sound and complete: see Theorem~\ref{th:3}.

\begin{theorem}\label{th:3}
Let $\mathcal{D}$ be a satisfiable temporal BAT axiomatizing a basic hybrid automaton $H$ as described above, 
let $\sigma \eq do([\alpha_1, \ldots, \alpha_n], S_0)$ be an executable ground situation term of $\mathcal{D}$, and $\tau$ a real number such that $\tau \geq start(\sigma)$, i.e., $\tau$ is a time point after the last action $\alpha_n$ of $\sigma$. Then
\begin{align*}
&\mathcal{D} \models{} Init(Q(S_0), X_\text{init}(S_0)) \land{} \\
&\hphantom{\mathcal{D} \models{}} (\forall a, s,t)\big[do(a,s) \!\sqsubseteq\! \sigma \land start(s) \lleq t \lleq time(a) \lor{}\\
&\hphantom{\mathcal{D} \models{}} s \eq \sigma \land start(\sigma) \leq t \leq \tau \big] \to Inv(Q(s), X(s,t))
\end{align*}
if and only if a finite trajectory of $H$ can be uniquely constructed from $\sigma$ and $\tau$. 
\end{theorem}

\begin{proof} (Sketch.)
%
Fix an arbitrary model $\mathcal{M}$ of $\mathcal{D}$ and consider the sequence $\eta \eq  \langle \Delta_i, q_i, \nu_i\rangle_{i\in [1,n+1]}$ such that
\begin{itemize}
\item $\Delta_j \eq (time(\alpha_j)\!-\!start(\sigma_{j}))^\mathcal{M}$ for each $do(\alpha_{j}, \sigma_{j}) \sqsubseteq \sigma$ and $\Delta_{n+1} \eq \tau\!-\!start(\sigma)^\mathcal{M}$;
\item $q_j \eq Q(\sigma_j)^\mathcal{M}$  for each $do(\alpha_{j}, \sigma_{j}) \sqsubseteq \sigma$, $q_{n+1} \eq Q(\sigma)^\mathcal{M}$;
\item $\nu_j = flow^\mathcal{M}(q_j,X\init(\sigma_{j})^\mathcal{M},t)$ for each $do(\alpha_{j}, \sigma_{j}) \sqsubseteq \sigma$ and $\nu_{n+1} = flow^\mathcal{M}(q_{n+1},X\init(\sigma)^\mathcal{M},t)$.
\end{itemize}
Clearly, $\eta$ is finite. Since $\sigma$ is executable, $time(\alpha_i) \ggeq start(\sigma_i)$, so the property (a) of a trajectory is satisfied, and $\Sigma_{i\in I}\Delta_i$ is also finite. By the precondition of $trans$, we have (e). Assuming $ \mathcal{M} \models Init(Q(S_0),X\init(S_0))$, we have (d). By (d) and the SSA for $Q$, we have (b). (c) follows from the remainder of the premise and the SEA for $X$. Thus, $\eta$ is a finite trajectory of $H$. Conversely, if $\eta$ is a finite trajectory of $H$, then by (d) we have $\mathcal{M} \models Init(Q(S_0),X\init(S_0))$, and by (c)--(e), $\mathcal{M}$ entails the remainder of the SC expression.
\end{proof}
Clearly, this axiomatization of HA is a very special case of a temporal BAT. It rules out any non-trivial queries about the content of its states because its discrete states are a finite set and there are no objects, let alone relations on objects. 
A general temporal BAT does not have this limitation.




We conclude this section by observing that, while classic hybrid automata \cite{Nerode2007} are based on a finite representation of states and atomic state transitions (apart from the continuous component), richer representations began to attract the interest of the hybrid system community. Of particular interest is the work by Platzer \cite{Platzer2010,Platzer2012LAHS,PlatzerJAR2017} 
based on first-order dynamic logic \cite{Pratt76,HaKT00} extended to handle  differential equations for describing continuous change. The work presented here contributes to this line of research by providing a very rich representation of the discrete states described relationally through the richness of Situation Calculus. 
\cbstart
Both \cite{PlatzerQDDL2012} and our paper propose to go beyond
hybrid automata with a finite number of states to hybrid systems where states may have
complex structure. The main difference between these two approaches is in the availability of
situation terms. As a consequence, the usual SC-based reasoning tasks \cite{reiter} 
can be easily formulated in our temporal BATs.
\cbend 

\section{Conclusion}\label{s:conclusion}

Inspired by hybrid systems, we have proposed a temporal extension of SC with a clean distinction between atemporal fluents, responsible for transitions
between states, and new functional fluents with a time argument, representing 
continuous change within a state.


In this paper we focused on semantics. 
However the connection with Hybrid Systems established here opens new perspectives for future work on automated reasoning as well.

In hybrid systems, the practical need for robust specification and
verification tools for hybrid automata resulted in the development of
a multitude of logic-based approaches. An in-depth overview of logics
for analyzing hybrid systems is given in \cite{davoren}.  More
recently, the results from
\cite{GaoAvigadClarkeLICS2012,GaoAvigadClarkeIJCAR2012} show that
that certain classes of decision
problems belong to reasonable complexity
classes.  These results provide foundations for verification of
robustness in hybrid systems
\cite{KongGaoChenClarkeTACAS2015}.
Platzer's work offers some decidability results for verification based on quantifier eliminations \cite{Platzer2010,PlatzerQDDL2012,PlatzerJAR2017}.
\cbstart
Note that quantified differential dynamic logic, 
the variant of first-order dynamic logic studied in \cite{PlatzerQDDL2012}, which focuses on functions and does not allow for arbitrary relations on objects,  cannot encode  situation calculus action theories in an obvious way. For example, it includes only one low-level primitive action, namely assignment, but the BATs provide agent actions that can be used to model a system on a higher level of abstraction. 
\cbend
Nevertheless, it may be interesting to study the reductions of fragments of Golog \cite{LevesqueRLLS97} and basic action theories with or without continuous time to such a dynamic logic, to exploit existing \cite{PlatzerQDDL2012} and future decidability results. 

On the other hand, while research in hybrid systems focuses on solving certain verification problems, the present paper, thanks to regression over situations, proposes an approach to solve other reasoning problems that cannot be formulated in hybrid systems. Moreover, the recent work on bounded theories \cite{DeGiacomoLesperancePatriziAIJ16,CalvaneseGMP18} provides promising means to study decidable cases in the realm of situation calculus, which could be of interest to hybrid systems as well.
This conceptual interaction between hybrid systems and situation calculus is an interesting avenue for future work.

\section*{Acknowledgement}
Thanks to the Natural Sciences and Engineering Research Council of Canada for
financial support.


\bibliography{refs.bib}
\bibliographystyle{aaai}

\onecolumn
\appendix

\section{Proof of Theorem 1}

\subsection*{Preliminaries}
\begin{itemize}
\item[] \emph{Temporal change axiom}. For $1 \leq i \leq k$ for some $k \geq 1$:
\begin{align*}
\gamma_i(\bar{x},s) \land \delta_i(\bar{x},y,t,s) \to f(\bar{x},t,s)\eq y.\tag{TCA}
\end{align*}

\item[] \emph{Property of a well-defined TCA}:
\begin{align*}
\gamma_i(\bar{x},s) \to \exists y\, \delta_i(\bar{x},y,t,s).\tag{WDP}
\end{align*}

\item[] \emph{Positive normal form change axiom}:
\begin{align*}\tag{PNFCA}
\Phi(\bar{x},y,t,s) \to f(\bar{x},t,s)\eq y,
\end{align*}
where $\Phi(\bar{x},y,t,s)$ is $\bigvee_{i=1}^k \gamma_i(\bar{x},s) \land \delta_i(\bar{x},y,t,s)$.

\item[] \emph{Consistency axiom}:
\begin{align*}\tag{Cons}
\Phi(\bar{x},y,t,s) \land \Phi(\bar{x},y',t,s) \to y\eq y'.
\end{align*}

\item[] \emph{Explanation closure axiom}:
\begin{align*}\tag{ECA}
f(\bar{x},start(s), s) \nneq f(\bar{x}, t, s) \to \exists y \, \Phi(\bar{x},y,t,s).
\end{align*}
\end{itemize}

\begin{lemma}\label{exclusive}
If $\mathcal{D} \models$ Cons, then a set $S$ of temporal change axioms (TCA) can be syntactically transformed into another set $S'$ of TCA in which any two distinct contexts $\gamma_{a}(\bar{x},s)$, $\gamma_{b}(\bar{x},s)$ are mutually exclusive wrt $\mathcal{D}$.
\end{lemma}

\begin{proof}
Suppose the set of TCA contains two axioms
\begin{align}
\label{eq:tca1}&\gamma_a(\bar{x},s) \land \delta_a(\bar{x},y,t,s) \to f(\bar{x},t,s)\eq y,\\
\label{eq:tca2}&\gamma_b(\bar{x},s) \land \delta_b(\bar{x},y,t,s) \to f(\bar{x},t,s)\eq y.
\end{align}
with $a \neq b$ such that $\mathcal{D} \models \exists \bar{x} \exists s (\gamma_a(\bar{x},s) \land \gamma_b(\bar{x},s))$. By Cons, 
\begin{align*}
&\gamma_a(\bar{x},s) \land \delta_a(\bar{x},y,t,s) \land \gamma_b(\bar{x},s) \land \delta_b(\bar{x},y',t,s) \to y\eq y',
\end{align*}
i.e., whenever the premises of both axioms are satisfied, they must agree on $y$. Thus, we can replace equations (\ref{eq:tca1}), (\ref{eq:tca2}) with a logically equivalent (wrt $\mathcal{D} \land \text{Cons}$) set of new temporal change axioms
\begin{align*}
&[\gamma_a(\bar{x},s) \land \neg \gamma_b(\bar{x},s)] \land \delta_a(\bar{x},y,t,s) \to f(\bar{x},t,s)\eq y,\\
&[\neg \gamma_a(\bar{x},s) \land \gamma_b(\bar{x},s)] \land \delta_b(\bar{x},y,t,s) \to f(\bar{x},t,s)\eq y,\\
&[\gamma_a(\bar{x},s) \land \gamma_b(\bar{x},s)] \land \delta_a(\bar{x},y,t,s) \to f(\bar{x},t,s)\eq y,
\end{align*}
whose contexts are strictly mutually exclusive wrt $\mathcal{D}$. By repeatedly applying this process to each pair of TCA whose contexts are simultaneously satisfiable, we obtain a set of TCA whose contexts are pairwise mutually exclusive wrt $\mathcal{D}$.
\end{proof}

Henceforth, we assume that all contexts of TCA are pairwise mutually exclusive.

\addtocounter{theorem}{-4}
\begin{theorem}
If, for each TCA, the background theory $\mathcal{D}$ entails WDP, then the axioms PNFCA, ECA in the models of Cons are logically equivalent to
\begin{align*}
\begin{split}
f(\bar{x},t,s)\eq y \liff [\Phi(\bar{x},y,t,s) \lor y \eq f(\bar{x},start(s), s) \land \neg \Psi(\bar{x},s)],
\end{split}
\end{align*}
where $\Psi(\bar{x},s)$ denotes $\bigvee_{1\leq i \leq k}\gamma_i(\bar{x}, s)$, a disjunction of all contexts.
\end{theorem}

\subsection*{Derivation}

Unless otherwise noted, we assume that all object assignments interpret the variables $\bar{x},t,s$ arbitrarily and identically.

\begin{enumerate}[label=\Roman*.]

\item\label{item:nnfca}
\emph{Negative normal form change axiom}:
\begin{align*}
f(\bar{x},t,s) \eq y \to \neg\exists z(\Phi(\bar{x},z,t,s) \land y \nneq z). \tag{NNFCA}
\end{align*}
Cons $\land$ PNFCA $\models$ NNFCA.
\begin{proof}

NNFCA holds in a model $\mathcal{M}$ under an object assignment $\sigma$ iff, for every choice of object $A \in \mathbb{R}$, 
\begin{align*}
&\text{either }\mathcal{M},\sigma(y \mapsto A) \models f(\bar{x},t,s)\nneq y \tag*{Case 1}\\
&\text{or }\text{there is no $B \neq A$ in $\mathbb{R}$ such that  }\mathcal{M},\sigma(y \mapsto B) \models \Phi(\bar{x},y,t,s). \tag*{Case 2}
\end{align*}

Take arbitrary $\mathcal{M},\sigma$ such that $\mathcal{M},\sigma \models$ Cons $\land$ PNFCA. Consider Cons. It holds in $\mathcal{M}$ under $\sigma$ iff there is at most one object $Y \in \mathbb{R}$ such that $\mathcal{M},\sigma(y \mapsto Y) \models \Phi(\bar{x},y,t,s)$. If such $Y$ does not exist, then NNFCA is satisfied through Case 2. Otherise, If $Y$ exists and $\mathcal{M}, \sigma \models$ PNFCA, then $\mathcal{M}, \sigma(y \mapsto Y) \models f(\bar{x},t,s) \eq y$. Since $(f)^{\mathcal{M}}$ is a function, for all $Z \nneq Y$, $\mathcal{M}, \sigma(y \mapsto Z) \models f(\bar{x},t,s) \nneq y$, so NNFCA is satisfied through Case 1. Since, for all $Z \nneq Y$, $\mathcal{M},\sigma(y \mapsto Z) \models \neg \Phi(\bar{x},y,t,s)$, NNFCA is satisfied through Case 2.
\end{proof}

\item
In the models of Cons and PNFCA, ECA is equivalent to either of
\begin{align*}
&f(\bar{x},start(s), s) \eq y \land f(\bar{x}, t, s) \nneq y \to \exists z (\Phi(\bar{x},z,t,s) \land y \nneq z),\tag{ECA1}\\
&f(\bar{x},start(s), s) \nneq y \land f(\bar{x}, t, s) \eq y \to \Phi(\bar{x},y,t,s).\tag{ECA2}
\end{align*}
\begin{proof}
ECA is satisfied in an arbitrary model $\mathcal{M}$ under an arbitrary assignment $\sigma$ if and only if either $f^\mathcal{M}(\sigma(\bar{x}), start^\mathcal{M}(\sigma(s)), \sigma(s))$ and $f^\mathcal{M}(\sigma(\bar{x}), \sigma(t), \sigma(s))$ coincide, or they are distinct (implying $start^\mathcal{M}(\sigma(s)) \nneq \sigma(t)$) and there exists some $Y \in \mathbb{R}$ such that $\mathcal{M}, \sigma(y \mapsto Y) \models \Phi(\bar{x},y,t,s)$. In the former case, ECA1 and ECA2 are trivially satisfied as well. Let us consider the latter case.

Cons is satisfied in exactly those models $\mathcal{M}'$ were there is at most one object $Y' \in \mathbb{R}$ such that $\mathcal{M}', \sigma'(y\mapsto Y') \models \Phi(\bar{x},y,t,s)$.

PNFCA is satisfied in exactly those models $\mathcal{M}''$ where, for all $Y'' \in \mathbb{R}$, if $\mathcal{M}'',\sigma''(y\mapsto Y'') \models \Phi(\bar{x},y,t,s)$, then $\mathcal{M}'',\sigma''(y\mapsto Y'') \models f(\bar{x},t,s)\eq y$.

NNFCA is satisfied in the models of Cons $\land$ PNFCA (see \ref{item:nnfca} above).

Let $T_1 \nneq T_2 \in \mathbb{R}$ such that $start^\mathcal{M}(\sigma(s)) \eq T_1$ and $\sigma(t) = T_2$. Let $Y_1 \in \mathbb{R}$ such that $Y_1 = f^\mathcal{M}(\sigma(\bar{x}), start^\mathcal{M}(\sigma(s)), \sigma(s))$. Let $Y_2 \in \mathbb{R}$ such that $Y_2 = f^\mathcal{M}(\sigma(\bar{x}), \sigma(t), \sigma(s))$. Recall that $Y_1 \neq Y_2$. Let $\sigma^\star$ denote the assignment $\sigma(t_1 \mapsto T_1, t_2 \mapsto T_2, y_1 \mapsto Y_1, y_2 \mapsto Y_2, y \mapsto Y)$.

We need to show that ECA1 and ECA2 are satisfied if and only if there exists some $Y \in \mathbb{R}$ such that $\mathcal{M},\sigma^\star \models f(\bar{x},t_1,s) \eq y_1 \land f(\bar{x},t_2,s) \eq y_2 \land \Phi(\bar{x},y,t_2,s)$. 

\emph{Only if}: By PNFCA, $\mathcal{M},\sigma^\star \models f(\bar{x},t_2,s)\eq y$, and since $(f)^\mathcal{M}$ is a function, 
$Y$ and $Y_2$ must coincide. Then $\mathcal{M},\sigma^\star \models f(\bar{x},t_1,s) \nneq y_2 \land f(\bar{x},t_2,s) \eq y_2 \land \Phi(\bar{x},y_2,t_2,s)$ and, since $Y_1 \nneq Y_2$, $\mathcal{M},\sigma^\star \models f(\bar{x},t_1,s) \eq y_1 \land f(\bar{x},t_2,s) \nneq y_1 \land \exists y\, \Phi(\bar{x},y,t_2,s)$.

\emph{If}: From ECA1, there exists some $Z$ distinct from $Y_1$ such that $\mathcal{M},\sigma^\star(z \mapsto Z) \models \Phi(\bar{x},z,t_2,s)$. By PNFCA and Cons, $Z \eq Y$, so $\mathcal{M},\sigma^\star \models f(\bar{x},t_1,s) \eq y_1 \land f(\bar{x},t_2,s) \eq y_2 \land \exists y\, \Phi(\bar{x},y,t_2,s)$. From ECA2, $\mathcal{M}, \sigma^\star \models \Phi(\bar{x},y,t_2,s)$. By PNFCA, $\mathcal{M}, \sigma^\star \models f(\bar{x},t_2,s) \eq y$ and $\mathcal{M}, \sigma^\star \models f(\bar{x},t_1,s) \nneq y$, so $\mathcal{M},\sigma^\star \models f(\bar{x},t_1,s) \eq y_1 \land f(\bar{x},t_2,s) \eq y_2 \land \exists y\, \Phi(\bar{x},y,t_2,s)$.
\end{proof}

\item
Syntactic variant of ECA1:
\begin{align*}
f(\bar{x},start(s), s) \eq y \land \neg \exists z (\Phi(\bar{x},z,t,s) \land y \nneq z) \to f(\bar{x}, t, s) \eq y. \tag{ECA1*}
\end{align*}

\item
PNFCA $\land$ ECA1* is equivalent to
\begin{align}\label{left}
\Phi(\bar{x}, y, t, s) \lor [f(\bar{x},start(s),s) \eq y \land \neg \exists z(\Phi(\bar{x},z,t,s) \land y \nneq z)] \to f(\bar{x},t,s) \eq y.
\end{align}

\item
Syntactic variant of ECA2:
\begin{align*}
f(\bar{x},t,s) \eq y \to \Phi(\bar{x},y,t,s) \lor f(\bar{x},start(s),s) \eq y.\tag{ECA2*}
\end{align*}

\item 
In the models of Cons, ECA2* $\land$ NNFC is equivalent to
\begin{align}\label{right}
f(\bar{x},t,s) \eq y \to \Phi(\bar{x},y,t,s) \lor [f(\bar{x},start(s),s) \eq y \land \neg \exists z(\Phi(\bar{x},z,t,s) \land y \nneq z)].
\end{align}
\begin{proof} ECA2* $\land$ NNFC $\equiv$
\begin{align*}
&f(\bar{x},t,s) \eq y \to [\Phi(\bar{x},y,t,s) \lor f(\bar{x},start(s),s) \eq y] \land \neg\exists z(\Phi(\bar{x},z,t,s) \land y \nneq z)\\
&\equiv f(\bar{x},t,s) \eq y \to [\Phi(\bar{x},y,t,s) \land \neg\exists z(\Phi(\bar{x},z,t,s) \land y \nneq z)] \lor{}\\
&\hspace{8.06em} [f(\bar{x},start(s),s) \eq y \land \neg\exists z(\Phi(\bar{x},z,t,s) \land y \nneq z)]
\end{align*}
In the models of Cons, $\Phi(\bar{x},y,t,s) \land \neg\exists z(\Phi(\bar{x},z,t,s) \land y \nneq z) \equiv \Phi(\bar{x},y,t,s)$, yielding RHS.
\end{proof}

\item
(\ref{left}) $\land$ (\ref{right}) is equivalent to
\begin{align*}
f(\bar{x},t,s) \eq y \liff \Phi(\bar{x},y,t,s) \lor [f(\bar{x},start(s),s) \eq y \land \neg \exists z(\Phi(\bar{x},z,t,s) \land y \nneq z) ]. \tag{SEA1}
\end{align*}

\item
In the models of WDP and Cons, SEA1 is equivalent to
\begin{align*}
f(\bar{x},t,s) \eq y \liff \Phi(\bar{x},y,t,s) \lor [f(\bar{x},start(s),s) \eq y \land \neg \bigvee_{i=1}^k \gamma_i(\bar{x},s)]. \tag{SEA2}
\end{align*}
\begin{proof}

Take arbitrary $\mathcal{M},\sigma$ such that $\mathcal{M},\sigma \models \Phi(\bar{x},y,t,s) \lor [f(\bar{x},start(s),s) \eq y \land \neg \exists z(\Phi(\bar{x},z,t,s) \land y \nneq z)]$. All TCA in $\Phi$ are mutually exclusive wrt $\mathcal{D}$ (see Lemma \ref{exclusive}). Three cases are possible.
\begin{itemize}
\item $\mathcal{M},\sigma \models \gamma_i(\bar{x},s) \land \delta_i(\bar{x},y,t,s)$ for exactly one $1 \leq i \leq k$ and $\mathcal{M},\sigma \models f(\bar{x},start(s),s) \nneq y$. In this case, the right-hand side of the biequivalence SEA2 is also satisfied.
\item $\mathcal{M},\sigma \models \neg \gamma_i(\bar{x},s) \lor \neg \delta_i(\bar{x},y,t,s)$ for all $1 \leq i \leq k$, so $\mathcal{M},\sigma \models f(\bar{x},start(s),s) \eq y \land \neg \exists z(\Phi(\bar{x},z,t,s) \land y \nneq z)$. None of the TCA are in effect. The models of WDP forbid the case where $\mathcal{M},\sigma \models \gamma_i(\bar{x},s)$ but $\mathcal{M},\sigma \models \neg \delta_i(\bar{x},y,t,s)$, meaning that, for all $1 \leq i \leq k$, we have $\mathcal{M},\sigma \models \neg \gamma_i(\bar{x},s)$, so the right-hand side of the biequivalence SEA2 is satisfied.
\item $\mathcal{M},\sigma \models \gamma_i(\bar{x},s) \land \delta_i(\bar{x},y,t,s)$ for exactly one $1 \leq i \leq k$ and $\mathcal{M},\sigma \models f(\bar{x},start(s),s) \eq y$. This is a boundary case between the previous two: one TCA is in effect, but the value of $y$ which it computes for $\bar{x},t,s$ coincides with the initial value of the fluent at $f$. In this case, the right-hand side of the biequivalence SEA2 is also trivially satisfied.
\end{itemize}
Now consider the opposite direction, from SEA2 to SEA1. Take arbitrary $\mathcal{M},\sigma$ such that $\mathcal{M},\sigma \models \Phi(\bar{x},y,t,s) \lor [f(\bar{x},start(s),s) \eq y \land \neg \bigvee_{i=1}^k \gamma_i(\bar{x},s)]$. If $\mathcal{M},\sigma \models \Phi(\bar{x},y,t,s)$, then the right-hand side of SEA1 is satisifed. If $\mathcal{M},\sigma \models f(\bar{x},start(s),s) \eq y \land \neg \bigvee_{i=1}^k \gamma_i(\bar{x},s)$, then $\mathcal{M},\sigma \models \neg \Phi(\bar{x},y,t,s)$, so $\mathcal{M},\sigma \models \neg \exists z(\Phi(\bar{x},z,t,s) \land y \nneq z)]$, and the right-hand side of SEA1 is satisifed.
\end{proof}
\end{enumerate}

\end{document}